\documentclass{article}

\usepackage[utf8]{inputenc}

\usepackage{color}
\usepackage{amsmath}
\usepackage{amssymb}
\usepackage{mathtools}
\usepackage{amsthm}
\usepackage{geometry}
\usepackage{hyperref}
\usepackage{cleveref}
\usepackage[colorinlistoftodos]{todonotes}
\usepackage{dsfont}
\usepackage{accents}
\usepackage{mathrsfs}
\usepackage{authblk}

\theoremstyle{plain}
\newtheorem{theorem}{Theorem}[section]
\newtheorem{lemma}[theorem]{Lemma}
\newtheorem{definition}[theorem]{Definition}
\newtheorem{corollary}[theorem]{Corollary}
\newtheorem{proposition}[theorem]{Proposition}

\theoremstyle{remark}
\newtheorem*{rem*}{Remark}
\newtheorem{remark}[theorem]{Remark}

\binoppenalty=\maxdimen
\relpenalty=\maxdimen

\hypersetup{bookmarksnumbered}

\newcommand{\EE}{\mathbb{E}}
\newcommand{\PP}{\mathbb{P}}
\newcommand{\R}{\mathbb{R}}
\newcommand{\N}{\mathbb{N}}

\newcommand{\Radon}{\mathcal{R}}

\newcommand{\supp}{\operatorname{supp}}

\newcommand{\Indicator}{\mathds{1}}

\newcommand{\NN}{\mathcal{NN}}
\newcommand{\lebesgue}{\boldsymbol{\lambda}}

\newcommand{\IIDsim}{\overset{iid}{\sim}}

\renewcommand{\emptyset}{\varnothing}

\numberwithin{equation}{section}

\title{Dimension-independent learning rates for high-dimensional classification problems}

\author[2]{Andres Felipe Lerma-Pineda}
\author[2]{Philipp Petersen}
\author[1,2]{Simon Frieder}
\author[1,3]{Thomas Lukasiewicz}
\affil[1]{Department of Computer Science, University of Oxford, UK}
\affil[2]{Department of Mathematics, University of Vienna, Austria}
\affil[3]{Institute of Logic and Computation, Vienna University of Technology, Vienna, Austria}

\begin{document}

\maketitle
\begin{abstract}
We study the problem of approximating and estimating classification functions that have their decision boundary in the $RBV^2$ space. Functions of $RBV^2$ type arise naturally as solutions of regularized neural network learning problems and neural networks can approximate these functions without the curse of dimensionality. 
We modify existing results to show that every $RBV^2$ function can be approximated by a neural network with bounded weights. 
Thereafter, we prove the existence of a neural network with bounded weights approximating a classification function. 
And we leverage these bounds to quantify the estimation rates.
Finally, we present a numerical study that analyzes the effect of different regularity conditions on the decision boundaries.
 \end{abstract}



 
\section{Introduction}

Neural networks (NNs) have shown exceptional performance for highly demanding tasks that take long periods of time and huge effort for humans \cite{lecun2015deep, goodfellow2016deep}.
One famous application area of neural networks is image classification, where the input dimensions of these NNs correspond to the number of pixels in the image, which is typically a large number. Therefore, one intriguing question is whether or not these NNs are subject to the curse of dimensionality. 
An approximation method is said to be subject to the curse of dimensionality if the performance of the method deteriorates exponentially when the dimension grows \cite{bellman1952theory, novak2009approximation}. 

In \cite{barron1993universal}, for example, it was proved that for a certain class of functions on Euclidean spaces---called the Barron class---the number of neurons required for a neural network to approximate an element of the Barron class does not grow exponentially with the dimension of the underlying space. 
More concretely, it is possible to approximate a given Barron function by a shallow NN (i.e., two-layer) with an arbitrary number of neurons $N$ and approximation error in the $L^{2}$ or $L^{\infty}$ norm of the order $N^{-1/2}$ \cite{ barron1992neural, barron1993universal}. 
Such results have been generalized and applied to the study of discontinuous functions as models for binary classification functions \cite{caragea2020neural, petersen2021optimal}. 
In this work, we continue with the study of learning discontinuous functions but now not associated with the class of Barron functions. Instead, we consider the space of $RBV^2$ functions \cite{parhi2022near}, which is another functional class that has been shown to be approximable without a curse of dimensionality.

The problem of learning discontinuous functions appears in many applications, such as, for example, \textit{classification of images}. A classification problem can be modeled via a function defined on an Euclidean space, which is usually called the \textit{classifier}. 
Here, we assume that a classifier is a function of the form $\sum_{k=1}^K c_k \mathds{1}_{\Omega_k}$ where $\Omega_k \subset \R^d$, $d \in \N$ are disjoint sets and $c_k$ is called the label of $\Omega_k$ for $k=1, \dots, K$. 
Examples of such labels are natural numbers $1,2, \dots, K$. This paper is concerned with the problem of approximating and estimating binary classifiers by NNs, i.e., $K=2$. 
Previous works on this topic can be found in \cite{caragea2020neural, imaizumi2019deep, imaizumi2022advantage, petersen2018optimal}. 
In all of these articles, different assumptions on the boundaries $\partial \Omega_k$ are imposed, and results on the approximability and estimability of the classifiers are discussed. 
Essentially, a more complex condition on the boundary induces a harder learning problem. 
Our work complements these results by studying a further assumption on the boundaries. 

We model the regularity of the decision boundaries by requiring them to be the graph of a regular function. This can be formalized through the concept of a horizon function. 
Horizon functions are binary functions defined on a compact subset of $\R^d$, $d \in \N$ taking values in $\{0,1\}$. In simple terms, for a function $f: \R^{d-1} \to \R$, an associated horizon function $h_f$ is given by $h_f(x) = \mathds{1}_{f(x^{[i]}) \leq x_i}$, where $x=(x_1,x_2, \dots, x_i \dots, x_d) \in \R^d$ and $x^{[i]}= (x_1, \dots, x_{i-1}, x_{i+1}, \dots, x_d)$ for some fixed $i$. For a more formal definition, we refer to Definition \ref{def:Hor}.

As mentioned before, we consider the case where $f$ belongs to $RBV^2$. In Section \ref{sec:approxOfRBV2Functions}, we introduce several notions that lead us to the concept of $RBV^2$ functions. 
For this space, we define the $RTV^2$ seminorm. 
Intuitively, the $RTV^2$ seminorm is a measure of the sparsity of the second derivatives of a function $f$ in the Radon domain. Under certain assumptions, shallow neural networks are solutions for the problem of minimizing a functional of squared data-fitting errors plus the $RTV^2$ seminorm \cite[Theorem 1]{parhi2022near}.  
More importantly, the problem of training a NN that minimizes this functional is equivalent to the problem of training a NN that minimizes squared data-fitting errors with weight decay, i.e., a regularization term proportional to the squared Euclidean norm of the NN's weights. 
The $RBV^2$ norm results from adding terms relating to the slope and value at $0$ of the underlying function. 

We summarize our findings for the framework outlined above in the following subsection.

\subsection{Our contribution}
Here, we present our main results on the approximation and estimation of horizon functions with the graph of an $RBV^2$ function $f$ as the decision boundary. 
A crucial role will be played by the $RBV^2$ norm (see Section \ref{sec:approxOfRBV2Functions}). 
In particular, the magnitude of the weights of the approximating NN of $f$ depends on the $RBV^2$ norm of $f$ denoted as $\Vert f \Vert_{RBV^2}$. For the purpose of this study, we only consider the domain to be the closed unit ball with center at the origin. This ball is denoted as $B_1^d$. When we refer to the $RBV^2$ space, the domain is always $B_1^d$ unless stated otherwise.  Now, we present the main results of this paper.

\paragraph*{Approximation of $RBV^2$ functions:} For every function $f \in RBV^2$, it was proven in \cite[Theorem 8]{parhi2022near}, that there is a shallow NN with $K \in \N$ neurons in the hidden layer that uniformly approximates $f$ with accuracy of the order of $K^{-(d+3)/(2d)}$. For our results in Section \ref{sec:Learning} to hold, we need a slightly stronger result. For Theorem \ref{thm:HorizonFunctionsApproximationTheorem} and \ref{thm:UpperBoundViaHingeLossMinimisation} to hold, we need a neural network approximating $f$ which has its weights bounded with a bound depending linearly on the $RBV^2$ norm $\Vert f \Vert_{RBV^2}$ of $f$. 
Thus, we have modified the main statement of \cite[Theorem 8]{parhi2022near} and completed the relevant missing steps of its proof to obtain Proposition \ref{coro:OurCor}, which is summarized below.

\begin{proposition}\label{coro:OurCor1}
    Let $d \in \mathbb{N}$, $d \geq 3$ and let $f \in RBV^2(B_1^d)$. Then, for every $N \in \mathbb{N}$,  there is a shallow $NN$ $f_N$ with $N$ neurons in the hidden layer and with weights bounded by a constant $C>0$ that depends linearly on $\Vert f \Vert_{RBV^2}$ such that 
\begin{align}\label{eq:esti1}
    \|f - f_N\|_{L^\infty(B_1^d)} \lesssim_d N^{-\frac{d+3}{2d}}.
\end{align}
\end{proposition}
Equation \ref{eq:esti1} tells us that the rate does not depend exponentially on the dimension $d$ as the exponent is $-1/2-3/(2d)$. However, the exponent grows slowly when $d \to \infty$, which worsens the approximation rate. Notice that as $d \to \infty$, the approximation error behaves like $N^{-1/2}$, which is the rate for the Barron class proved in \cite{barron1993universal}. This approximation rate does not explode when $d$ grows. Thus, NNs overcome the \textit{curse of dimensionality} for the class of $RBV^2$ functions.\\
We prove Proposition \ref{coro:OurCor1} in Section \ref{sec:approxOfRBV2Functions}, but we now provide a general overview of its proof. We use an integral representation of every $RBV^2$ function. Indeed, every $f \in RBV^2(B_1^d)$ can be represented as an integral term over the domain $\mathbb{S}^{d-1} \times [-1,1]$ plus an affine linear function. The integral term is given by
\begin{eqnarray}\label{eq:Term}
    \int_{\mathbb{S}^{d-1} \times [-1,1]} \varrho(w^\top z - b)d\mu(w, b), \qquad \text{ for all }  z \in B_1^d 
\end{eqnarray}
where $\mu$ is a  measure on $\mathbb{S}^d \times [-1,1]$ and $\varrho$ is the ReLU activation function. We first consider the case where  $z \in \mathbb{S}^d$ and assume that $\mu$ is a probability measure, prove some intermediate results, and then extend the argument to $z \in B_1^d$.
We can see that \eqref{eq:Term} can be expressed as the sum of three integral terms. In two out of three terms, the integral
\begin{equation}\label{eq:II}
    \int_{\mathbb{S}^{d-1} \times [-1,1]} \vert w^\top z - b \vert d \mu(w,b),  \qquad \text{ for all }  z \in B_1^d 
\end{equation}
appears as a factor but with different probability measures that will be denoted as $\mu_1$, $\mu_2$. We prove that this integral can be approximated by a NN thanks to \cite[Theorem 1]{matouvsek1996improved}.
However,  that result requires the domain of integration to be a $d$-dimensional sphere. To solve this problem, we make a change of variables to the Equation \ref{eq:II}, which is presented in Lemma \ref{lem:ChangeInteg}. 
Then, in Proposition \ref{prop:NNfor2}, we finally show that we can approximate the integral \eqref{eq:II} by a NN for all $z \in t \mathbb{S}^d$, $t>0$. In Lemma \ref{lemma:int}, we prove that the result holds even when the integral representation \eqref{eq:II} is defined for all $z \in B_1^d$.  
Finally, by showing that the factor \eqref{eq:II} can be approximated by shallow NNs, we prove in Lemma \ref{lemma:int} that \eqref{eq:Term} can be approximated by a NN, which in turn leads to Theorem \ref{thm:OurTheo}. In Theorem \ref{thm:OurTheo}, we state that every $RBV^2$ function where the integral representation \eqref{eq:fext} has $\mu$ as a probability measure can be approximated by a shallow NN. 
Thereafter, the general case where $\mu$ is an arbitrary measure is proved in Proposition \ref{coro:OurCor}. 

\paragraph*{Approximation of horizon functions associated to a $RBV^2$ functions:} Based on Proposition \ref{coro:OurCor}, we derive the following theorem on the existence of a NN that approximates a horizon function associated with an $RBV^2$ function. To state the result, we need the concept of a tube-compatible measure introduced in \cite[Section 6]{caragea2020neural}. These are measures such that in every $\epsilon$-neighborhood of every curve, the mass of the neighborhood scales like $C \epsilon^{\alpha}$ for constants $C>0$ and $\alpha>0$.

\begin{theorem}\label{thm:HorizonFunctionsApproximationTheoremIntro}
  Let $d \in \N_{\geq 2}$, $N \in \N$, $q,C > 0$, and $\alpha \in (0,1]$. Further let $h_f$ be a horizon function associated to $f \in \{g \in RBV^2(B_1^{d-1}) \colon \Vert g \Vert_{RBV^2} \leq q\}$. Then, there exists a NN $I_N$ with two hidden layers
  such that for each tube compatible measure $\mu$ with parameters $\alpha,C$, we have
 \begin{align} \label{eq:thisWeWantToProveIntro}
    \mu
    (
      \{
        x \in B_1^d : h_f (x) \neq I_N(x)
      \}
    )
    \lesssim_d N^{-\alpha \frac{d+3}{2d}}.
 \end{align}
  Moreover, $I_N$ has at most $d+N+5$ neurons and at most $(d+3)N + 2d+ 11$ non-zero weights.
  The weights (and biases) of $I_N$ are bounded in magnitude by $\mathcal{O}(N^{\frac{d+3}{2d}})$ for $N\to \infty$.
\end{theorem}
The proof of this theorem is organized as follows:
\begin{enumerate}
    \item We note that $h_f(x)= H(x_i - f(x^{[i]}))$, where $H$ is the Heaviside function.
    \item We use Proposition \ref{coro:OurCor} to uniformly approximate $f \in RBV^2$ by a NN $f_N$ up to an error of $N^{-(d+3)/(2d)}$. 
    \item We prove that the Heaviside function $H(x) = \mathds{1}_{(0, \infty)}$ can be approximated by a NN $H_{\delta}$ for all $\delta>0$ such that the function $H_{\delta}$ and the Heaviside function $H$ differ only on the interval $(-\delta, \delta)$.
    \item Choosing $\delta = N^{-(d+3)/(2d)}$, we observe that $h_f(x) \neq H_{\delta}(x_i - f_N(x^{[i]}))$ only for $x$ outside of a $2 \delta$ strip around the decision boundary $f(x^{[i]}) = x^i$.
    \item Finally, we prove \eqref{eq:thisWeWantToProveIntro} (cf. the corresponding Equation \ref{eq:thisWeWantToProve} from Section \ref{sec:Learning}) using the properties of a tube-compatible measure.
\end{enumerate}

\paragraph*{Upper bounds on learning:} Finally, our approximation results lead us to the problem of estimating a horizon function associated with a function $f$ when a training set $S= (x_i, y_i)_{i=1}^m$, $m \in \N$ is given. In Section \ref{sec:Learning}, we analyze the performance of the standard empirical risk minimization procedure, where the loss function is the \text{Hinge loss} and the hypothesis set is a certain class of ReLU NNs. Our result on learning is Theorem \ref{thm:UpperBoundViaHingeLossMinimisation}: If we consider as hypothesis set the set of NNs with two layers and at most $N = m^{2d/(3d+3)}$ neurons, we prove that for all $\kappa>0$, any minimizer $\phi_{m,S}$ for a training set $S$ has a risk of at most 
$\mathcal{O}(m^{- \frac{d+3}{3d+3}+ \kappa})$.
The proof of Theorem \ref{thm:HorizonFunctionsApproximationTheoremIntro} (cf. the corresponding Theorem \ref{thm:HorizonFunctionsApproximationTheorem} from Section \ref{sec:Learning}) is similar to that of \cite[Theorem 5.7]{petersen2021optimal}.
This final result is studied numerically in Section \ref{sec:numeric}. 
As there are similar results when one changes the smoothness conditions on $f$, we could ask ourselves how the smoothness of $f$ affects the learnability of the function in practice. 
To answer that question, we compare the test error after training NNs for different conditions on the function $f$, which amounts to assuming $f$ is an element of a ball with respect to various norms or seminorms. 
We verify numerically that functions in a Barron-norm or $RBV^2$ norm ball can be learned better by NNs than functions in $L^{\infty}$ or $L^1$ balls.

\subsection{Related work}

Our results concern two central aspects of learning problems. First, the approximation and learning without the curse of dimensionality, and second, the approximation and learning of discontinuous classifier functions. We will review the work related to these two themes below.

\paragraph*{Approximation and estimation of functions by shallow neural networks:}
It has been shown that shallow NNs can break the curse of dimensionality.
This was widely discussed in \cite{barron1993universal}. 
In fact, the approximation problem for functions with one finite Fourier moment---called Barron functions---is not affected by the underlying dimension. 
In \cite{caragea2020neural, ma2020towards, siegel2021characterization, wojtowytsch2020representation}, different extensions of the notion of a Barron function are discussed. 
The space of $RBV^2$ functions is closely related to the problem of function approximated by shallow NNs as well (see \cite{barron1993universal}). 
This space is associated with the problem of estimating a function from a set of samples when a regularization term is added to the loss function. 
Shallow NNs also break the curse of dimensionality in the approximation problem for the space of $RBV^2$ functions (see \cite{parhi2022near}).
Several other function classes for which the curse of dimensionality can be overcome by deep NNs instead of shallow NNs have been proposed, such as the class of composition of low dimensional functions \cite{poggio2016and, shaham2018provable, nakada2020adaptive, cloninger2021deep, pineda2023deep}, bandlimited functions \cite{montanelli2019deep} and solutions of some high-dimensional PDEs \cite{han2018solving, hutzenthaler2020overcoming, jentzen2018proof}. Strictly speaking, there is a dependence in the approximation rate on the ambient dimension for these classes of functions, but such dependency is usually polynomial.

\paragraph*{Approximation of discontinuous functions by neural networks:}
There are different approaches to the problem of classification. In \cite{wojtowytsch2020priori}, the authors study the problem of classification associated with two different sets $C_+$ and $C_-$. It is assumed that the distance between these two sets is positive. Under such conditions, some approximation and estimation bounds by shallow NNs for the classification problem are presented. 
One may consider the classifier to be a function of the form $\sum_{k=1}^K f_k \mathds{1}_{\Omega_k}$, where each $\Omega_k \subset \R^d$ has a piecewise smooth boundary $\partial \Omega$ and $f_k: \R^d \to \R$ are smooth functions. The analysis for the problems of analyzing and estimating such classifiers is discussed in \cite{imaizumi2019deep, imaizumi2022advantage, petersen2018optimal}. Therein, it is proved that the approximation and estimation rates are strongly determined by the regularity of the boundary $\partial \Omega_k$ when the functions $f_k$ are sufficiently smooth. 
In \cite{caragea2020neural}, the same approach is followed, but a different condition on $\partial \Omega_k$ is imposed. The boundaries are assumed to be locally parametrized by Barron functions. Indeed, this idea has inspired our own work, and some of the ideas of \cite{caragea2020neural} are similar to the ones we discussed here. 

\paragraph*{Approximation of discontinuous functions by other approaches:} To conclude this subsection, we mention some non-deep-learning-based techniques that have been applied to the approximation of piecewise smooth functions. Although indicator functions belong to the set of piecewise constant functions, it turns out that the set of piecewise constant functions and the set of piecewise smooth functions admit the same approximation rates \cite{petersen2018optimal}. Various representation methods have been applied to approximate piecewise smooth functions. Shearlets and curvelet systems achieve (almost) optimal $N$-term approximation rates for $C^2(\R^2)$ functions with $C^2$ jump curves (see \cite{candes2004new, candes1999curvelets, guo2007optimally, kutyniok2011compactly, voigtlaender2017analysis}). For the study of approximation of two-dimensional functions with jump curves smoother than $C^2$, the bandelet system is proposed in \cite{le2005sparse}. This bandelet system is made up of a set of properly smoothly-transformed boundary-adapted wavelets optimally adapted to the smooth jump curves. 
Finally, a different representation system called surflets (see \cite{chandrasekaran2004compressing}) yields optimal approximation of piecewise smooth functions. This system is constructed by using a partition of unity as well as local approximation of horizon functions.

\subsection{Structure of the paper}
In Section \ref{sec:NN} we introduce several definitions related to NNs. 
The space of functions implemented by ReLU NN with $L$ layers, at most $N$ neurons and $W$ non-zero weights bounded by $B$, which is Definition \ref{def:NeuralNetworkSets} 
is particularly important in Section \ref{sec:Learning}. In Section \ref{sec:approxOfRBV2Functions}, we provide a formal definition of the space of $ RBV^2$ functions. 
We first define this space on an Euclidean space, and we then present the notion of the $RBV^2$ space on a bounded subset. 
We prove that each $RBV^2$ function can be efficiently uniformly approximated by a shallow NN in Proposition \ref{coro:OurCor}.
In Section \ref{sec:Learning}, we study the learning problem of estimating a horizon function associated with a $RBV^2$ function from a sample. We present upper bounds for the risk of the minimizer empirical risk when we train our NN with the Hinge loss in Theorem \ref{thm:UpperBoundViaHingeLossMinimisation}. 
Finally, in Section \ref{sec:numeric}, we study numerically our results on learning and compare the practical learning rates for various regularity conditions on the decision boundary. 

\section{Neural Networks}\label{sec:NN}

Although there are different types of NNs, for this study, we restrict ourselves to the well-known type of feed-forward NNs. We collect several useful definitions that pave the way for our theoretical results, which are provided in the definition below. The NN formalism underlying this definition was introduced in \cite[Definition 2.1]{petersen2018optimal}.

\begin{definition}
  Let $d,L\in\N$.
  A \emph{neural network (NN) $\Phi$ with input dimension $d$ and $L$ layers}
  is a sequence of matrix-vector tuples
  \[
    \Phi = \bigl( (A_1,b_1), (A_2,b_2), \dots, (A_L,b_L) \bigr),
  \]
  where, for $N_0 = d$ and $N_1, \ldots, N_L \in \N$,
  each $A_\ell$ is an $N_\ell \times N_{\ell-1}$ matrix, and $b_\ell \in \R^{N_\ell}$.\\
  For a NN $\Phi$ and an activation function $\sigma : \R \to \R$,
  we define the associated \emph{realization of the NN $\Phi$} as
  \[
    R_\sigma\Phi : \quad
    \R^d \to \R^{N_L} , \quad
    x \mapsto x_L = R_\sigma \Phi(x),
  \]
  where the output $x_L \in \R^{N_L}$ results from the scheme
  \begin{align*}
    x_0      &\coloneqq x \in \R^d, \\
    x_{\ell} &\coloneqq \sigma \left(A_{\ell} \, x_{\ell-1} + b_\ell\right) \in \R^{N_\ell}
    \quad \text{ for } \ell = 1, \dots, L-1,\\
    x_L      &\coloneqq A_{L} \, x_{L-1} + b_{L} \in \R^{N_L}.
  \end{align*}
  Here $\sigma$ is understood to act coordinate-wise.
  We call $N(\Phi) \coloneqq d+\sum_{j=1}^L N_j$ the \emph{number of neurons} of $\Phi$,
  $L(\Phi) \coloneqq L$ the \emph{number of layers}, and $W(\Phi) \coloneqq \sum_{j=1}^L (\|A_j\|_{0}+\|b_j\|_{0})$
  is called the \emph{number of weights} of $\Phi$.
  Here, $\| A \|_0$ and $\| b \|_0$ denote the number of non-zero entries
  of the matrix $A$ and the vector $b$, respectively.
  Moreover, we refer to $N_L$ as the \emph{output dimension} of $\Phi$.
  The activation function $\varrho : \R \to \R, x \mapsto \max\{0, x\}$ is called the \emph{ReLU}.
  We call $R_\varrho \Phi$ a \emph{ReLU neural network}.
  Finally, the vector $(d, N_1, N_2, \dots, N_L) \in \N^{L+1}$ is called
  the \emph{architecture} of $\Phi$.
\end{definition} 
 As an independent definition, we introduce the set of NNs with fixed $L$, $W$, and $d$ and an upper bound $B>0$ on the modulus of the weights. 
\begin{definition}
\label{def:NeuralNetworkSets}
Let $d \in \N_{\geq 2}$, $N, W , L\in \N$, and $B > 0$. We denote by $\NN(d, L, N, W, B)$ the set of ReLU NNs where the underlying NNs have $L$ layers, at most $N$ neurons per layer, and at most $W$ non-zero weights. We also assume that the weights of the NNs are bounded in absolute value by $B$. Moreover, we set
   \[
    \NN_\ast (d, L, N, W, B)
    \coloneqq \big\{
         f \in \NN(d, L, N, W, B)
         \quad\colon\quad
         0 \leq f(x) \leq 1 \quad \text{ for all } \, x \in [0,1]^d
       \big\}.
  \]
\end{definition}

\section{Approximation of \texorpdfstring{$RBV^2$}{} functions}\label{sec:approxOfRBV2Functions}
In \cite{parhi2022kinds}, the authors demonstrate that shallow NNs are an optimal ansatz system for solving the estimation problem of a function $f \in RBV^2$ from a sample $S$. In this section, we give some insight into this $RBV^2$ space and present some of its properties. 
Our main result, Proposition \ref{coro:OurCor} demonstrates that every function $f$ in $RBV^2$ with $f(0)=0$ can be approximated by a NN with weights bounded depending on a norm of the function only. 
The bounds on the weights are crucial to show learning bounds in Section \ref{sec:Learning}. 
We point out that the existence of an approximating NN was already shown in \cite[Theorem 8]{parhi2022kinds}. 
That result, however, does not include any control of the weights.

\subsection{The set of $RBV^2$ functions}
 We start this subsection with some useful definitions related to the notion of the $RBV^2$ space. 
 For an extensive survey on the class of Radon-regular functions, we refer to \cite{parhi2022near}, where all the definitions and results of this subsection have been taken from. 
 We first define the $RBV^2$ space when the domain is the Euclidean space $\R^d$ and continue to restrict its domain to a bounded subset $\Omega$. At this stage, it is necessary to introduce the notion of Radon transform.

\begin{definition}\label{def:radonT}
Let $d \in \mathbb{N}$. For a function $f: \mathbb{R}^d \to \mathbb{R}$, we define its  \emph{Radon transform} as the function $\Radon f: \mathbb{S}^{d-1} \times \R \to \R$, 
\begin{equation*}
    \Radon f(w,s) \coloneqq \int_{ \{x: w^\top x = s \}} f(x) dS(x),
\end{equation*}
if such an integral exists, where $(w,s) \in \mathbb{S}^{d-1} \times \mathbb{R}$ and $d S$ denotes the surface integral on the domain $\{x: w^\top x = s \}$. 
\end{definition}
We next introduce the concept of the \textit{ramp filter}, which appears in the inversion formula of the Radon transform. 
\begin{definition}\label{def:rampf}
Let $d \in \mathbb{N}$. The \emph{ramp filter} $\Lambda^{d-1}$ is defined as 
$$
    \Lambda^{d-1} \coloneqq (- \partial_t^2)^{(d-1)/2},
$$
where $\partial_t^2$ is the partial derivative with respect to $t$.
\end{definition}
 The ramp filter helps us to define the second-order total variation of a given function $f$. Herein, $\mathcal{M}(X)$ denotes the space of signed Borel measures defined on a set $X$. 
\begin{definition}\label{def:RTVNorm} 
    
Let $d \in \mathbb{N}$. We defined the \emph{second-order total variation} of a function $f: \mathbb{R}^d \to \R$ as 
\begin{align*}
    RTV_{\R^d}^2(f) \coloneqq \frac{1}{2(2\pi)^{d-1}} \|\partial_t^2 \Lambda^{d-1} \Radon (f) \|_{\mathcal{M}(\mathbb{S}^{d-1} \times \R)}, 
\end{align*}
where $\|\cdot\|_\mathcal{M}$ denotes the total variation norm. 
\end{definition}

The second-order total variation of a function is not a norm but a seminorm. It can be made into a norm for a Banach space when other terms are added. We first define the associated space, $RBV^2$, and then the norm. 

\begin{definition}\label{def:RBV2}
Let $d \in \mathbb{N}$. The \emph{$RBV^2$ space} is defined as the set 
\begin{align*}
    RBV^2(\mathbb{R}^d) \coloneqq \{f \in L^{\infty, 1} \colon RTV_{\R^d}^2(f)< \infty \}, 
\end{align*}
where the space $ L^{\infty, 1}$ contains all functions $f$ such that $\sup_{x \in \R^d}|f(x)|(1+\|x\|_2)^{-1}< \infty$. 
\end{definition}
We now introduce a norm on $RBV^2(\R^d)$ that turns it into a Banach space. 

\begin{definition} \label{def:RBV^2norm}
Let $d \in \mathbb{N}$. For every function $f \in RBV^2(\R^d)$ we define its $RBV^2$ norm as 
\begin{equation*}
    \Vert f \Vert_{RBV^2(\R^d)} \coloneqq RTV_{\R^d}^2(f) + |f(0)| + \sum_{k=1}^d |f(e_k)-f(0)|,
\end{equation*}
where $\{e_k \}_{k=1}^d $ denotes the canonical basis of $\R^d$.  
\end{definition}

Note that the point evaluations in Definition \ref{def:RBV^2norm} are well defined since $f$ is guaranteed to be Lipschitz continuous by \cite[Lemma 2.11.]{parhi2022kinds}.

After this crucial definition, our next step is to define the space of $RBV^2$ functions on a bounded domain $\Omega$. As we see below, the space $RBV^2(\Omega)$ can be defined in terms of the elements of $RBV^2(\R^d)$. 

\begin{definition} \label{def:RBV2bound}
Let $d \in \mathbb{N}$ and let $\Omega \subset \R^d$. We define the $R BV^2(\Omega)$ space as     
\begin{equation*}
RBV^2(\Omega) \coloneqq  \{f \in \mathcal{D}'(\Omega) \colon \exists g \in RBV^2(\R^d) \text{ s.t } g|_{\Omega} = f\}, 
\end{equation*}
where $\mathcal {D}'(\Omega)$ denotes the space of distributions on  $\Omega$. Moreover, we define
\begin{equation*}
    RTV^2_{\Omega}(f) \coloneqq \inf_{g \in RBV^2(\R^d) \colon f = g|_{\Omega}} RTV_{\R^d}^2(g)
\end{equation*}
and 
\begin{equation*}
    \Vert f \Vert_{RBV^2(\Omega)} \coloneqq \inf_{g \in RBV^2(\R^d) \colon f = g|_{\Omega}} \Vert g \Vert_{RBV^2(\Omega)}.
\end{equation*}
\end{definition}

One can prove that for a given $f \in RBV^2(\Omega)$, there is a function $f_{ext} \in RBV^2 (\R^d)$ that admits an integral representation and has the property that $f_{ext}|_{\Omega} = f$. This is shown in the following lemma.
\begin{lemma}[{\cite[Lemma 2]{parhi2022near}}]\label{lem:fext}
    Let  $d \in \mathbb{N}$, $\Omega \subset \R^d$ be a bounded set, and let $\varrho$ be the ReLU activation function. For each $f \in RBV^2(\Omega)$, there is a function $f_{ext} \in RBV^2(\R^d)$ such that
    \begin{equation*}
        f_{ext}(x) = \int_{\mathbb{S}^{d-1} \times \R} \varrho( w^\top x - b) d\mu(w, b) + c^\top x + c_0,
    \end{equation*}
    for all $x \in \R^d$, where $\mu \in \mathcal{M}(\mathbb{S}^{d-1} \times \R)$ and $\supp (\mu) \subset Z_{\Omega}$, where the set $Z_{\Omega}$ is the closure of
      \begin{equation*}
        \{z=(w,b) \in \mathbb{S}^{d-1} \times \R \colon \{ x: w^\top x=b \} \cap \Omega \neq \emptyset \},
    \end{equation*}
    and $f_{ext}|_{\Omega} = f $. Moreover, $RTV^2_\Omega(f) = RTV_{\R^d}^2(f_{ext}) = \Vert \mu \Vert_{\mathbb{S}^{d-1} \times \R}$.
\end{lemma}
 In practice, we may encounter functions defined on arbitrary subsets $\Omega \subset \R^d$. For simplicity, all the results of this paper assume that the domain of the function $f$ is the unit ball with center at the origin.  This unit ball is denoted by $\Omega = B_{1}^d(0) = B_1^d$. In addition, we denote $RTV^2(f) = RTV_{B_1^d}^2$ and $\Vert f \Vert_{RBV^2} = \Vert f \Vert_{RBV^2(B_1^d)}$.
 
\begin{remark}\label{rmk:fbound}
    In the setting of Lemma \ref{lem:fext} with $\Omega = B_{1}^d $, it is shown in \cite[Lemma 2 and Remark 4] {parhi2022near} that
    \begin{equation*}
        Z_{\Omega} = \mathbb{S}^{d-1} \times [-1,1],
    \end{equation*}
    and therefore, every function $f \in RBV^2(B_1^d)$ admits a representation 
    \begin{equation}\label{eq:fext}
        f(x) = \int_{ \mathbb{S}^{d-1} \times [-1,1]} \varrho( w^\top x - b) d\mu(w, b) + c^\top x+ c_0,
    \end{equation}
    for all $x\in B_{1}^d$.    
    Notice that, since $0, e_1, \dots, e_d \in B_{1}^d$ for every extension $f_{ext}$ of $f$, it holds that $\Vert f_{ext} \Vert_{RBV^2(\R^d)} = \Vert f \Vert_{RBV^2}$.
    We use this last equality to derive bounds for $\Vert c \Vert_{\infty}$ and $\vert c_0 \vert$ in terms of $\Vert f \Vert_{RBV^2}$. Let us assume that $\Vert f \Vert_{RBV^2} = C$ for $C>0$ and $f(0)=0$. Therefore,
    \begin{equation*}
        \int_{ \mathbb{S}^{d-1} \times [-1,1]} \varrho( - b) d\mu(w, b) +  c_0 = 0.
    \end{equation*}
    Clearly, $0 \leq \varrho(-b) \leq 1$. 
    Thus, $-C \leq \int_{ \mathbb{S}^{d-1} \times [-1,1]} \varrho( - b) d\mu(w, b) \leq C$. 
    Hence, $|c_0| \leq C$. 
    Now, as the canonical basis $\{e_k \}_{k=1}^d \subset B_1^d$, we can use \eqref{eq:fext} to compute $f(e_k)$ for all $k=1,2, \dots, d$. 
    Notice that $|f(e_k)|\leq \Vert f \Vert_{RBV^2}\leq C$ which leads us to 
     \begin{equation*}
        -C \leq \int_{ \mathbb{S}^{d-1} \times [-1,1]} \varrho(w_k - b) d\mu(w, b) + c_k+ c_0 \leq C         
    \end{equation*}
    because $w^Te_k-b = w_k-b$ where $w_k$ is the $k$-th coordinate of the vector $w$. 
    Clearly, $-2 \leq w_k - b \leq 2$ and since by Lemma \ref{lem:fext} $\Vert \mu \Vert_{\mathbb{S}^{d-1} \times \R} = RTV^2(f) \leq \Vert f \Vert_{RBV^2}$, we conclude 
    $$
    -2C \leq \int_{ \mathbb{S}^{d-1} \times [-1,1]} \varrho(w_k - b) d\mu(w, b) \leq 2C.
    $$ 
    Then, $-3C\leq \int_{ \mathbb{S}^{d-1} \times [-1,1]} \varrho(w_k - b) d\mu(w, b)+c_0 \leq 3C$ holds and in turn $-4C \leq c_k \leq 4C$. 
    Thus, we obtain that $\Vert c \Vert_{\infty} \leq 4\Vert f \Vert_{RBV^2}$ and $\vert c_0 \vert \leq \Vert f \Vert_{RBV^2} $. 
    
\end{remark}

\subsection{Approximation of $RBV^2$ functions by neural networks}

In the remainder of this section, we prove that $f \in RBV^2(B_1^d)$ can be approximated by a NN with bounded weights. This is an adapted version of \cite[Theorem 8]{parhi2022near}. 
\begin{proposition}
    
\label{coro:OurCor}
 Let $d \in \mathbb{N}$, $d \geq 3$ and let $f \in RBV^2(B_1^d)$. Then, for every $N \in \mathbb{N}$,  there are $w_k \in \R^d$, $ b_k, v_k \in \R$   $c\in \R^d$ and $c_0 \in \R$ with 
\begin{align*}
    |v_k|, |b_k|, |c_0|, \|c\|_2, \Vert w_k \Vert_2 \leq 5\Vert f \Vert_{RBV^2},
\end{align*} 
where $k =1,2, \dots, N$, such that for 
$$
    f_N(x) \coloneqq \sum_{k=1}^N v_k \varrho(w_k^\top x - b_k) + c^\top x + c_0, \text{ for } x \in B_1^d,
$$
it holds that 
\begin{align*}
    \|f - f_N\|_{L^\infty(B_1^d)} \lesssim_d \Vert f \Vert_{RBV^2} N^{-\frac{d+3}{2d}}.
\end{align*}
In particular, we have that $f_N \in \NN(d, 2, N+2, (2+d)(N+2) + 1, 5 \Vert f \Vert_{RBV^2}).
$
\end{proposition}
Before we prove Proposition \ref{coro:OurCor}, we spend some time stating and proving several other statements that pave the way for its proof. 
Our starting point for the proof of Proposition \ref{coro:OurCor} is the integral representation of $g$ as in Equation \ref{eq:fext}. 
Due to the properties of the absolute value function, we can decompose this integral as the sum of other integrals. This will become evident in the proof of Theorem \ref{thm:OurTheo}. At this stage, we focus our attention on the integral term
 \begin{eqnarray}\label{eq:IntForm}
     \int_{\mathbb{S}^{d-1} \times [-1,1]} | w^\top z - b|d \mu(w, b), \text{ for } z \in B_1^d. 
 \end{eqnarray}
We begin implementing a change of variable for the integral \eqref{eq:IntForm}, which in turn allows us to prove the existence of a NN.  We can easily see that the region of integration $\mathbb{S}^d \times [-1,1]$ is a cylinder with axis parallel to the axis $z_{d+1}$ of $\R^{d+1}$, radius $r=1$ and lower and upper boundaries at $z_{d+1}= -1$ and $z_{d+1}=1$, respectively. 
Let us first assume that $z \in \mathbb{S}^{d-1}$. Notice that the integral term
 \begin{equation}\label{eq:Int2}
     \int_{\mathbb{S}^{d-1} \times [-1,1]} | w^\top z - b| d \mu(w, b), \text{ for } z \in \mathbb{S}^d,
 \end{equation}
can be transformed into an equal integral by multiplying and dividing the integrand function by the term $ \sqrt{\Vert w \Vert^2 + b^2} = \sqrt{1+b^2}$ as follows
 \begin{equation}\label{eq:Int3}
\int_{\mathbb{S}^{d-1} \times [-1,1]} \sqrt{1+b^2} \left (\dfrac{| w^\top z - b|}{\sqrt{1+b^2}} \right) d \mu(w, b).
 \end{equation}

The process to transform such an integral into another integral, the domain of which is contained in the sphere $\mathbb{S}^d \subset \R^{d+1}$ is shown in Lemma \ref{lem:ChangeInteg}.
\begin{lemma}\label{lem:ChangeInteg}
   Let $d \in \R$, $z\in \mathbb{S}^d$, $w = (w_1, w_2, \dots, w_d)^\top \in \mathbb{S}^{d-1}$, and $b \in [-1,1]$. Further, let $v = (v_1,v_2,\dots,v_{d+1})^\top$ be defined as follows:
     $$
        v_i \coloneqq \dfrac{w_i}{\sqrt{1+b^2}},
     $$
    for $i=1,2,\dots,d$ and
        $$
            v_{d+1} \coloneqq \dfrac{b}{\sqrt{1+b^2}}.
        $$
     Moreover, let $\phi:  \mathbb{S}^{d-1}\times [-1,1] \to \mathbb{S}^{d}$ be the function defined as $\phi(w, b)= (v_1,v_2, \dots , v_{d+1})^\top$. Then, for all $z \in \mathbb{S}^d$ we have
 
    \begin{equation*}
        \int_{\mathbb{S}^{d-1} \times [-1,1]}| w^\top z - b|d \mu(w, b) = 
\int_{\mathbb{S}^{d}}  \dfrac{1}{\sqrt{1-v_{d+1}^2}} \left| v^\top \widetilde{z} \right | \chi_B(v) d \phi^{\ast}\mu(v), 
    \end{equation*}
    where $\widetilde{z} = (z,-1)^\top \in \R^{d+1}$, $B = \phi(\mathbb{S}^{d-1} \times [-1,1]) \subset \mathbb{S}^d \subset \R^{d+1}$, $\chi_B$ is the characteristic function of $B$, and $\phi^* \mu$ denotes the push-forward of the measure $\mu$.
\end{lemma}
\begin{proof}
The integral term in \eqref{eq:Int2} can be reorganized as Equation \ref{eq:Int3} which leads to
$$
\int_{\mathbb{S}^{d-1} \times [-1,1]} \sqrt{1+b^2} \left (\dfrac{| w^\top z - b|}{\sqrt{1+b^2}} \right ) d \mu(w, b) = \\ \int_{\mathbb{S}^{d-1} \times [-1,1]}  \sqrt{1+b^2} \left ( \dfrac{| \langle (w,b) , (z, -1) \rangle |}{\sqrt{1+b^2}} \right ) d \mu(w, b). \\
$$
Besides, it is clear that
$$
b^2 = \dfrac{v_{d+1}^2}{1-v_{d+1}^2},
$$ 
and 
$$
1+b^2 = 1+\dfrac{v_{d+1}^2}{1-v_{d+1}^2}= \dfrac{1}{1-v_{d+1}^2}.
$$
After applying \textit{change of variables} (see for example \cite[Section 19, Corollary 1]{bauer2001measure}), we obtain 
$$
\int_{\mathbb{S}^{d-1} \times [-1,1]} \sqrt{1+b^2} \left (\dfrac{| w^\top z - b|}{\sqrt{1+b^2}} \right ) d \mu(w, b)  = \int_{\mathbb{S}^{d}}  \dfrac{1}{\sqrt{1-v_{d+1}^2}} \left | v^\top \widetilde{z} \right | \chi_B(v) d \phi^{\ast}\mu(v),
$$
where $\widetilde{z}$ and $v_i$ are defined as stated in the assumptions.
\end{proof}

\begin{remark}\label{remark:4points}
Under the assumptions of Lemma \ref{lem:ChangeInteg}, we make the following observations.
\begin{enumerate}
    \item  Since $b \in [-1,1]$, we have that $2b^2 \leq b^2+1$ and hence ${b^2}/(b^2+1) \leq 1/2$. Taking the square root yields that 
    $$ |v_{d+1}| \leq \dfrac{1}{\sqrt{2}}.$$
    Hence, $B$ is a subset of the set of elements in the sphere for which the last coordinate lies between $-1/\sqrt{2}$ and $1/\sqrt{2}$.

\item We know that the function $h: B \subset \mathbb{S}^{d} \to \mathbb{R}$ given by $h(v)=(1-v_{d+1}^2)^{-1/2}$ is continuous and thus integrable. This holds because the $\sigma$-algebra on $\mathbb{S}^{d}$ is that of the Borelian sets, which is the restriction of the Borel $\sigma$-algebra of $\mathbb{R}^{d+1}$. In addition, $h$ is a function bounded by $\sqrt{2}$. Indeed, from the previous step, we have
$v_{d+1}^2 \leq \dfrac{1}{2}$, 
and hence
$1-v_{d+1}^2 \geq 1/2$ 
and 
$$
\dfrac{1}{\sqrt{1-v_{d+1}^2}} \leq \sqrt{2}.
$$
\item If we denote by $C$ the following quantity
$$
C \coloneqq \int_{\mathbb{S}^{d}}  \dfrac{1}{\sqrt{1-v_{d+1}^2}}  \chi_B(v) d \phi^{\ast}\mu(v), 
$$
and by $\nu$ the following measure
$$ 
d \nu = \dfrac{1}{C \sqrt{1-v_{d+1}^2}}\chi_B(v) d \phi^{\ast}\mu,
$$
 we have that $\nu$ is a probability measure on the Borelian set of $ \mathbb{S}^d$. Then, we can apply \cite[Theorem 1]{matouvsek1996improved} for the integral
$$ 
\int_{\mathbb{S}^{d}}  \left | v^\top \widetilde{z} \right |d \nu(v),
$$ 
if $\widetilde{z} \in \mathbb{S}^d$. That is, for $\epsilon >0$, there is a set $Q = \{ v^{(1)},v^{(2)}, \dots, v^{(r)} \} \subset \mathbb{S}^d$ with $r \in \mathbb{N}$ and $r \leq C(d) \epsilon^{-2+6/(d+3)} $ for $C=C(d)>0$ a constant and $d\geq 3$, such that
$$
\left | \int_{\mathbb{S}^{d}}|v^\top z| d\nu(v) - \dfrac{1}{r}\sum_{i=1}^r |((v^{(i)})^\top z)|  \right| < \epsilon.
$$
\end{enumerate}
\end{remark}

We cannot apply Remark \ref{remark:4points} directly, because, even though $v=(v_1,\dots ,v_{d+1})^\top \in \mathbb{S}^d \subset \mathbb{R}^{d+1}$, it holds that $\widetilde{z} \notin \mathbb{S}^d$. Indeed, for $z \in \mathbb{S}^{d}$, the vector $\tilde{z}=(z, -1)$ is an element of $ \sqrt{2}\mathbb{S}^d$. However, the next result shows that the conclusions of Remark \ref{remark:4points} can be generalized for $z$ in a sphere of arbitrary radius.

\begin{proposition}\label{prop:NNfor2}
    Let $t>0$, $d \in \mathbb{N}$, and $ d\geq 3$. Then, for $\epsilon>0$, there is $\{ v^{(1)}, v^{(2)}, \dots, v^{(r)}  \} \subset \mathbb{S}^d $ where $r \in \mathbb{N}$, and $r \leq C(t, d) \epsilon^{-2+6/(d+3)}$ for $C=C(t,d)>0$ a constant depending on the dimension, such that for all $z \in t\mathbb{S}^d \subset \R^{d+1}$

    $$
\left | \int_{\mathbb{S}^{d}}|v^\top z| d\nu(v) - \dfrac{1}{r}\sum_{i=1}^r |((v^{(i)})^\top z)|  \right| < \epsilon.
$$

\end{proposition}

\begin{proof}
As $z \in t\mathbb{S}^d$, there is a $\widetilde{z} \in \mathbb{S}^d$ such that $z = t\widetilde{z}$. 
We now consider two cases. First, let us assume that $t>1$. By Remark \ref{remark:4points}, for $\epsilon>0$ there is a set $\{ v^{(1)}, v^{(2)}, \dots, v^{(r)}  \} \subset \mathbb{S}^{d}$ with $r \in \mathbb{N}$, $r \leq C(d) \epsilon^{-2+6/(d+3)}t^{2-6/(d+3)}$ such that 
$$
\left | \int_{\mathbb{S}^{d}}|v^\top \widetilde{z}| d\nu(v) - \dfrac{1}{r}\sum_{i=1}^r |(v^{(i)})^\top \widetilde{z}|  \right| < \epsilon/t.
$$
Thus,
$$
\left | \int_{\mathbb{S}^{d}}|v^\top t\widetilde{z}| d\mu(v) - \dfrac{1}{r}\sum_{i=1}^r |(v^{(i)})^\top t\widetilde{z}|  \right| < \epsilon.
$$
Now, if $t<1$
$$
\left | \int_{\mathbb{S}^{d}}|v^\top t\widetilde{z}| d\mu(v) - \dfrac{1}{r}\sum_{i=1}^r |(v^{(i)})^\top t\widetilde{z}|  \right| = t \left | \int_{\mathbb{S}^{d}}|v^\top z| d\mu(v) - \dfrac{1}{r}\sum_{i=1}^r |(v^{(i)})^\top z|  \right|  < t \epsilon < \epsilon.
$$
Therefore, in any case $r \leq C(d)\epsilon^{-2+6/(d+3)} \max\{1,t\}^{2-6/(d+3)} = C(d,t)\epsilon^{-2+6/(d+3)}$.
\end{proof}

\begin{remark}\label{remark:OurR}
    We would like to highlight that the previous result holds in particular for $t=\sqrt{2}$. As $(z,-1) \in \sqrt{2} \mathbb{S}^d$ when $z \in \mathbb{S}^d$, this case is especially relevant.  That being said, we proceed to approximate the term \eqref{eq:IntForm}. 
    For a given element $v^{(i)}\in B \subset \mathbb{S}^d$ as in Proposition \ref{prop:NNfor2}, we define $w^{(i)} = (v^{(i)}_1, v^{(i)}_2, \dots,v^{(i)}_d)$ and $b_i = v^{(i)}_{d+1}$ where $v^{(i)}_k$ denotes the $k-$th component of $v^{(i)}$ for all $i=1,2, \dots r$. 
    It is clear that $\sup_{k= 1, \dots, d+1} | w^{(i)}_k | \leq 1$ and $|b_i|\leq 1$. Therefore, for $z \in \mathbb{S}^d$
    $$
        \left | \int_{\mathbb{S}^{d}}|w^\top z - b| d\mu(v) - \dfrac{1}{r}\sum_{i=1}^r |(w^{(i)})^\top z-b_i|  \right| < \epsilon.
    $$

\end{remark} 

 Now, we are set to prove that there is a NN that approximates the integral term of Equation \ref{eq:fext}. 
 We first assume that $\mu$ of \eqref{eq:fext} is a probability measure. 

\begin{lemma}\label{lemma:int}
    Let $d \in \mathbb{N}$, $d\geq 3$ and let $\mu$ be a probability measure over $\mathbb{S}^{d-1} \times [-1,1]$. Then, let $f: B_d^1 \to \mathbb{R}$ be given by 
    \begin{equation*}
    f(z) = \int_{\mathbb{S}^{d-1} \times [-1,1]} \varrho(w^\top z-b) d\mu(w,b),
    \end{equation*}
    where $z \in \mathbb{S}^{d-1}$. Then, there exist vectors $w^{(1)}, w^{(2)}, \dots w^{(r)} \in \R^d$ and real numbers $b_1,b_2, \dots b_r \in \R$ with $r \in \mathbb{N}$, such that for all $z \in \mathbb{S}^{d-1}$
     \begin{equation*}
            \left \vert \int_{\mathbb{S}^{d-1} \times [-1,1]} |w^\top  z - b| d \mu(w,b) -   \dfrac{1}{r} \sum_{i=1}^r \varrho((w^{(i)})^\top  z-b_i) - \varrho(-(w^{(i)})^\top  z+b_i)  \right \vert < \epsilon,
    \end{equation*}
    where $\sup_{k= 1, \dots, d+1} | w^{(i)}_k | \leq 1$ and $|b_i|\leq 1$, $i=1,2, \dots r$, where $r \leq C(d) \epsilon^{-2+6/(d+3)}$.
    
\end{lemma}

\begin{proof}
    By Remark \ref{remark:OurR} there are $w^{(i)} \in [-1,1]^d$ and $b_i \in [-1,1]$ such that
    $$
        \left | \int_{\mathbb{S}^{d-1} \times [-1,1]}|w^\top z - b| d\mu(v) - \dfrac{1}{r}\sum_{i=1}^r |(w^{(i)})^\top  z-b_i|  \right| < \epsilon.
    $$    
   As $\vert x \vert = \varrho(x) - \varrho(-x)$, we have
        \begin{equation}
            \dfrac{1}{r} \sum_{i=1}^r |(w^{(i)})^\top  z - b_i| = \dfrac{1}{r} \sum_{i=1}^r \varrho((w^{(i)})^\top  z-b_i) - \varrho(-(w^{(i)})^\top  z+b_i), 
        \end{equation}
    and hence
        \begin{equation*}
            \left \vert \int_{\mathbb{S}^{d-1} \times [-1,1]} |w^\top  z - b| d \mu(w,b) -   \dfrac{1}{r} \sum_{i=1}^r \varrho((w^{(i)})^\top  z-b_i) - \varrho(-(w^{(i)})^\top  z+b_i)  \right \vert < \epsilon.
\end{equation*}

\end{proof}
 We now have all the ingredients to prove Theorem \ref{thm:OurTheo}. In the proof of the result, we use Lemma \ref{lem:fext} twice when \eqref{eq:fext} is expressed as the addition of three terms. We consider the case when $\mu$ is a probability measure. The case of general measures will be shown in a corollary thereafter.
 
\begin{theorem}\label{thm:OurTheo} Let $d \geq 3$ and $\mu \in \mathcal{M}(\mathbb{S}^{d-1} \times [-1,1])$ be a probability measure.  Let $g: B_1^d \to \R$ be the function defined as
\begin{equation*}
    g(z)=\int_{\mathbb{S}^{d-1} \times [-1,1]} \varrho(w^\top z - b) d \mu(w,b), \text{ for } z \in B_1^d.
\end{equation*}
Then, for all $\epsilon>0$, $g$ can be uniformly approximated with approximation error $\epsilon$ by a shallow NN 
\begin{equation*}
   \widetilde{f}(z) = \sum_{i=1}^K v_i\varrho((w^{(i)})^\top  z -b_{i}) + (w^{(0)})^\top z+b_0,
\end{equation*}
where $w^{(0)}, w^{(i)}\in \R^d$, $b_0, b_i \in \R$, $v_i \in \mathbb{R}$ and
$$
\Vert w^{(0)} \Vert_\infty , \Vert w^{(i)} \Vert_\infty , \vert v_i \vert, \vert b_i \vert, \vert b_0 \vert \leq 1,
$$
for all $i=1,2,\dots, K$ where $K \leq C(d) \epsilon^{-2+6/(d+3)}$.
\end{theorem}

\begin{proof}
Because the ReLU activation function can be expressed as $\varrho(x) = (x + |x|)/2$ and due to the following argument in \cite[Proposition 1]{bach2017breaking}, we have that
\begin{align}
    g(z) & =  \dfrac{1}{2} \int_{\mathbb{S}^{d-1} \times [-1,1]} (w^\top z -b) d \mu(w, b) + \dfrac{1}{2} \int_{\mathbb{S}^{d-1} \times [-1,1]} |w^\top z -b| d \mu(w, b) \nonumber \\  
        & =  \dfrac{1}{2} \int_{\mathbb{S}^{d-1} \times [-1,1]} (w^\top z -b) d \mu(w, b)\nonumber \\
        & \qquad + \dfrac{\mu_+(\mathbb{S}^{d-1} \times [-1,1])}{2} \int_{\mathbb{S}^{d-1} \times [-1,1]} |w^\top z -b| \dfrac{d \mu_+(w,b)}{\mu_+(\mathbb{S}^{d-1} \times [-1,1])} \nonumber \\
    & \qquad \qquad  - \dfrac{\mu_-(\mathbb{S}^{d-1} \times [-1,1])}{2} \int_{\mathbb{S}^{d-1}\times [-1,1]} |w^\top z -b| \dfrac{d \mu_-(w,b)}{\mu_-(\mathbb{S}^{d-1} \times [-1,1])},    
    \label{eq:import}
\end{align}
for $z \in \mathbb{S}^d$ and $\mu = \mu_+ + \mu_-$ is a Jordan decomposition. 
Furthermore, it was proved in Lemma \ref{lemma:int} that for $\epsilon>0$ and probability measure $\mu \in \mathcal{M}(\mathbb{S}^{d-1} \times [-1,1])$, there is an $r \in \mathbb{N}$ and $w^{(1)}, w^{(2)}, \dots, w^{(r)}\in \R^d$, $b_1, b_2, \dots, b_r \in [-1,1]$ such that
\begin{equation*}
    \left \vert \int_{\mathbb{S}^{d-1} \times [-1,1]} |w^\top  z - b| d \mu - \dfrac{1}{r} \sum_{i=1}^r \varrho((w^{(i)})^\top  z - b_i) - \varrho(-(w^{(i)})^\top z + c_i)  \right \vert < \epsilon,
\end{equation*}
for all $z \in B_1^d$ and $r \leq C(d) \epsilon^{-2+6/(d+3)}$ where $C=C(d)>0$ is a constant depending on the dimension. Due to the fact that 
\begin{equation*}
\dfrac{d \mu_+}{\mu_+(\mathbb{S}^{d-1} \times [-1,1])} \text{    and    } \dfrac{d \mu_-}{\mu_-(\mathbb{S}^{d-1} \times [-1,1])}
\end{equation*}
are probability measures, we can approximate each of the last two terms of Equation \ref{eq:import} according to Lemma \ref{lemma:int} as
\begin{equation*}
 \left \vert \int_{\mathbb{S}^{d-1} \times [-1,1]} |w^\top z - b| d \mu_{\pm } - \dfrac{\mu_{s, \pm}}{r} \sum_{i=1}^r \varrho((w^{(i)}_{\pm})^\top z - b_{\pm, i}) - \varrho(-(w^{(i)}_{\pm})^\top  z + b_{\pm, i})  \right \vert < \epsilon \mu_{s, \pm}  \leq \epsilon,
\end{equation*}
respectively, where $\mu_{s, \pm} = \mu_{\pm}(\mathbb{S}^{d-1} \times [-1,1]))\leq 1$ and $w^{(i)}_+$ and $b_{+,i}$ denotes the weights and biases when the measure involved is $\mu_{+}$ and $w^{(i)}_{-}$ and $b_{-,i}$ are defined likewise for $\mu_{-}$. 
Moreover, the first term of \eqref{eq:import} can be expressed as 
\begin{equation*}
    \dfrac{1}{2} \int_{\mathbb{S}^{d-1} \times [-1,1]} (w^\top z -b) d \mu(w, b) = (w^{(0)})^\top z - b_0,
\end{equation*}
where $\Vert w^{(0)} \Vert_{\infty} \leq 1$ and $\vert b_0 \vert \leq 1$. 
Setting 
\begin{align*}
    \widetilde{f}(z) & = (w^{(0)})^\top z - b_0 + \dfrac{\mu_{s, +}}{2r}\sum_{i=1}^r \varrho((w^{(i)}_{+})^\top z -b_{+,i}) - \varrho(-(w^{(i)}_{+})^\top z + b_{+,i}) \\
    & \qquad + \dfrac{\mu_{s, -}}{2r}\sum_{i=1}^r \varrho((w^{(i)}_{-})^\top z - b_{-, i}) - \varrho(-(w^{(i)}_{-})^\top z + b_{-,i}),
\end{align*}
and $K=4r$, we can rearrange $\widetilde{f}$ in such a way that
      $$
        \widetilde{f}(z) = \sum_{i=1}^K v_i \varrho((w^{(i)})^\top z -b_{i}) + (w^{(0)})^\top z+b_0,
      $$ 
      where $v_i $ is either $ \mu_{s, +}/(2r)$ or  $\mu_{s,-}/(2r)$ depending on whether $w^{(i)}$ was determined by $\mu_{s,+}$ or $\mu_{s,-}$. 
      Notice that, by construction $K \leq C(d) \epsilon^{-2+6/(d+3)}$ for $C=C(d)>0$ a constant.  
Additionally, all weights and biases are bounded in norm by 1,
$$
\vert v_i \vert, \vert b_i \vert, \vert b_0 \vert, \Vert w^{(i)} \Vert_{\infty}, \Vert w^{(0)} \Vert_{\infty} \leq 1
$$
for $i=1,2, \dots, K$. To conclude, notice that by construction $\Vert g - \widetilde{f} \Vert_{\infty} < \epsilon$.
\end{proof}

We now proceed by generalizing Theorem \ref{thm:OurTheo} to arbitrary measures.

\begin{corollary}\label{coro:PrevCor}

Let $d \in \mathbb{N}$, $d \geq 3$ and let $f \in RBV^2(B_1^d)$ with $f(0)=0$. 
Then, for all $\epsilon>0$, there exist $w^{(i)} \in \R^d$, $ b_i, v_i \in \R$, $c\in \R^d$ and $c_0 \in \R$ with 
\begin{align*}
    |v_i|, |b_i|, |c_0|, \|c\|_{\infty}, \Vert w^{(i)}\Vert_{\infty} \leq 5 \Vert f\Vert_{RBV^2},
\end{align*} 
where $i =1,2, \dots, K$ and $K \leq C(d) \|f\|_{RBV^2}^{2-6/(d+3)} \epsilon^{-2+6/(d+3)}$, such that $f$ can be uniformly approximated with approximation error $\epsilon$ by a shallow NN
$$
    f_K(x) = \sum_{i=1}^K v_i \varrho((w^{(i)})^\top x - b_i) + c^\top x + c_0, \quad \text{ for } x \in \R^d.
$$

\end{corollary}

\begin{proof}
   We assume first that $\Vert f \Vert_{RBV^2} = 1$. By Remark \ref{rmk:fbound}, we know that $f$ has an integral representation of the form
     $$
        f(x) = \int_ {\mathbb{S}^{d-1} \times [-1,1]} \varrho(w^\top x-b) d\mu(w,b) + \widetilde{c}^\top x+\widetilde{c}_0, \quad \text{ for all } x \in B_d^1,
     $$
     where $\Vert \widetilde{c} \Vert_{\infty} \leq 4$ and $\vert \widetilde{c}_0 \vert \leq 1$ and $\mu \in \mathcal{M}( \mathbb{S}^{d-1} \times [-1,1] )$ such that $\Vert \mu \Vert_{\mathbb{S}^{d-1} \times [-1,1]} = RTV^2(f)$. 
     
     If $RTV^2(f) = 0$ the result holds, since $f(x)= \widetilde{c}^T x$  in this case. If $RTV^2(f) \neq 0$, we define $\hat{f} \coloneqq f/RTV^2(f)$ and choose $\epsilon>0$. 

     According to Theorem \ref{thm:OurTheo}, the function 
     $$
     x \mapsto \frac{1}{RTV^2(f)}\int_ {\mathbb{S}^{d-1} \times [-1,1]} \varrho(w^\top x-b) d\mu(w,b)
     $$
     can be approximated by the realization of a NN 
     $$
        \widetilde{f}_K(x) = \sum_{i=1}^K \widetilde{v}_i \varrho((\widetilde{w}^{(i)})^\top x -\widetilde{b}_{i}) + (\widetilde{w}^{(0)})^\top x+\widetilde{b}_0, \text{ for } x \in \R^d
      $$ 
     with approximation error $\epsilon / RTV^2(f)$ and where $K \leq C(d) RTV^2(f) ^{2-6/(d+3)}\epsilon^{-2+6/(d+3)}$ for a constant $C=C(d)>0$ depending on the dimension $d$. This holds because the measure of the integrand $\nu = \mu/RTV^2(f)$ fulfills that $\Vert \nu \Vert_{\mathbb{S}^{d-1} \times [-1,1]}=1$. Moreover, for all $i=1,2, \dots, K$,
     $$
        \vert \widetilde{v}_i \vert, \vert \widetilde{b}_i \vert, \vert \widetilde{b}_0 \vert, \Vert \widetilde{w}^{(i)} \Vert_{\infty}, \Vert \widetilde{w}^{(0)} \Vert_{\infty} \leq 1.
     $$
     Let us define
     $$
          \widehat{f}_K \coloneqq  \widetilde{f}_K+ a^\top x+a_0.
       $$ 
     where $a= \widetilde{c}/RTV^2(f) $, $a_0 =  \widetilde{c}_0/RTV^2(f) $. This leads us to the fact that
      \begin{align*}
           \vert \widehat{f}(x) - \widehat{f}_K(x) \vert & = \left \vert  \int_ {\mathbb{S}^{d-1} \times [-1,1]} \varrho(w^\top x-b) d\nu(w,b) - \widetilde{f}_K(x)  \right \vert  < \epsilon/RTV^2(f)
      \end{align*}        
    for all $x \in B_1^d$.
    Finally, if we set     
    $$f_K(x) = RTV^2(f) \widehat{f}(x) = \sum_{i=1}^K v_i \varrho((w^{(i)})^\top x -b_{i}) + ( w^{(0)}+\widetilde{c})^\top x+ (b_0+\widetilde{c}_0), $$     
    where $w^{(0)}= RTV^2(f)\widetilde{w}^{(0)} $, $b_0= RTV^2(f)\widetilde{b}_0 $ and for all $i = 1,2, \dots, K$, $w^{(i)} = RTV^2(f)\widetilde{w}^{(i)}$, $v_i = RTV^2(f) \widetilde{v}_i$ and $b_i = RTV^2(f) \widetilde{b}_i$, it holds that 
    $$
    |f(x) - f_K(x)| = |RTV^2(f) \widehat{f}(x)- RTV^2(f) \widehat{f}_K(x)| = RTV^2(f) | \widehat{f}(x) - \widehat{f}_K(x) | < \epsilon.
    $$   
    We denote $c= w^{(0)} + \widetilde{c}$ and $c_0 = b_0+\widetilde{c}_0$ and observe that $|c_0| \leq |b_0|+| \widetilde{c}_0|  \leq 1+ \Vert f \Vert_{RBV^2} \leq 2$. 
    As $\Vert \widetilde{c} \Vert_{\infty} \leq 4 $, we conclude that $\Vert c \Vert_{\infty} \leq 5$. 
    In combination with the bounds for the weights in Theorem \ref{thm:OurTheo}, we obtain
    \begin{align*}
         |v_i|, |b_i|, |\widetilde{c_0}|, \|\widetilde{c}\|_2, \Vert w^{(i)} \Vert \leq 5.
    \end{align*}    
    For the general case where $\Vert f \Vert_{RBV^2} \neq 1$, we define
       $$
        \hat{f} \coloneqq \dfrac{f}{\Vert f \Vert_{RBV^2}},
       $$
       if $\Vert f \Vert_{RBV^2} \neq 0$ and apply the previous argument to conclude that for every $\epsilon>0$ there is a $NN$ such that $\Vert \widehat{f}-\widehat{f}_K \Vert < \epsilon / \Vert f \Vert_{RBV^2}$ where $K \leq C(d) \|f\|_{RBV^2}^{2-6/(d+3)} \epsilon^{-2+6/(d+3)}$ with all weights bounded by $5$. 
    If we denote  $f_K = \Vert f \Vert_{RBV^2} \widehat{f}_K$, we observe that $\Vert f-f_K \Vert_{\infty} < \epsilon.$  
    The boundedness of the weights follows immediately. If $\Vert f \Vert_{RBV^2} = 0$, $f=0$ and the result holds trivially.
    
    \end{proof}

 We now have all the results to prove Proposition \ref{coro:OurCor}. Essentially, we only need to express the approximation accuracy $\epsilon$ in terms of the number of neurons of a neural network.

 \begin{proof}[Proof of Proposition \ref{coro:OurCor}]
    For $N\in \N$, we set 
   $$
   \epsilon = \left(C(d) \|f\|_{RBV^2}^{2-6/(d+3)}/N\right)^{(d+3)/(2d)}.
   $$  
    Then, according to Corollary \ref{coro:PrevCor}, there is $\{w^{(1)}, \dots,w^{(r)}\} \subset [-1,1]^d$, $\{v_1, v_2, \dots, v_r \} \subset [-1,1]$ and $\{b_1, b_2, \dots, b_r \} \subset[-1,1]$ such that if we set
    $$f_r(x) = \sum_{i=1}^r v_i \varrho( (w^{(i)})^\top z-b_i) + c^Tx + c_0,$$ it holds that
    \begin{eqnarray*}
        |f(x) - f_K(x)  | = \left \vert \int_{\mathbb{S}^d} \varrho( w^\top x - b ) d\mu(w,b) -  \sum_{i=1}^r v_i \varrho( (w^{(i)})^\top x-b_i)  \right \vert < \epsilon,
    \end{eqnarray*}
    where $r \leq C(d) \|f\|_{RBV^2}^{2-6/(d+3)}
 \epsilon^{-2+6/(d+3)} = N$ for all $x \in B_1^d$. 
    We  set $w^{(r+1)}= \dots = w^{(N)} = 0 \in \R^d$, $v_{r+1} = \dots = v_N= 0 \in \R$ and $b_{r+1} = \dots = b_N= 0 \in \R$ and obtain  
    \begin{eqnarray*}
        \left \vert \int_{\mathbb{S}^d} \varrho( w^\top x - b ) d\mu(w,b) -  \sum_{i=1}^N v_i \varrho( (w^{(i)})^\top x-b_i ) \right \vert < \epsilon
    \end{eqnarray*}
    for all $x \in B_1^d$. Then, if
    $$
        f_K(x) = \sum_{i=1}^r v_i \varrho( (w^{(i)})^\top x-b_i) + c^Tx + c_0,
    $$ 
    the result follows.
    Finally, notice that $c^\top x = \varrho(c^Tx)-\varrho(-c^Tx)$. Hence, $f_N$ can be expressed as 
    $$
        f_N(x) = \sum_{i=1}^{N+2} v_i \varrho((w^{(i)})^\top x -b_{i}) + c_0,  
    $$ 
    and after counting the number of weights we conclude that 
    $$
        f_N \in \NN(d, 2, N+2, (2+d)(N+2)+1, 5 \Vert f \Vert_{RBV^2}).
    $$    
\end{proof}

\section{Upper bounds on learning horizon functions}\label{sec:Learning}

In this section, we achieve two main goals: we show that there is a NN that approximates horizon functions associated to $RBV^2$ functions and we present upper bounds for the corresponding learning problem. Let us formalize the notion of horizon functions associated to an arbitrary set of functions $\mathcal{B}$. These are sets of binary functions such that the discontinuity or the boundary between the classes can be parameterized as a regular function in all but one coordinate. 

\begin{definition}\label{def:Hor}
Let $d \in \N$, $d\geq 2$ and assume that $\mathcal{B} \subset C(B_1^{d-1}, \R)$. We define the \emph{set of horizon functions associated to $\mathcal{B}$} by 
\begin{align*}
    H_{\mathcal{B}} \coloneqq \{f_h^{i} = \mathds{1}_{x_i \leq h(x^{[i]})} \colon h \in \mathcal{B}, i \in [d]\},
\end{align*}
where $x^{[i]} \coloneqq (x_1, \dots, x_{i-1}, x_{i+1}, \dots, x_d)$ for $i \in [d]$. A function $f_h^{i} \in H_{\mathcal{B}}$ is called \emph{a horizon function with decision boundary described by $h \in \mathcal{B}$}.
\end{definition}

 In the sequel, we will analyze horizon functions with $RBV^2$ functions as the decision boundary, i.e., we choose $\mathcal{B} = RBV^2(B_1^{d-1})$.
\noindent Based on Proposition \ref{coro:OurCor}, we now produce an approximation of horizon functions with $RBV^2$ decision boundary by NNs.
We will compute the error of approximating Horizon functions via the 0-1 loss. 
To this end, we need to specify an underlying measure. 
This would typically be a uniform/Lebesgue measure, but at the very least, it needs to be \emph{tube-compatible}  (see \cite[Section 6]{caragea2020neural}).

\begin{definition}
Let $\mu$ be a finite Borel measure on $\R^d$. We say that $\mu$ is \emph{tube compatible with
parameters $\alpha \in (0, 1]$ and $C > 0$} if for each measurable function $f \colon \R^{d-1} \to \R$, each $i \in [d]$ and each $\epsilon \in (0, 1]$, we have
\begin{align}
    \mu(T_{f, \epsilon}^{i}) \leq  C \epsilon^{\alpha} \text{ where } T_{f, \epsilon}^{i} \coloneqq \{x \in \R^d \colon |x_i - f(x^{[i]})| \leq \epsilon\}, 
\end{align}
The sets $T_{f, \epsilon}^{i}$ are called \emph{tubes of width $\epsilon$ (associated to $f$)}.
\end{definition}
We now construct a NN that approximates horizon functions associated to $RBV^2$ functions. The proof is based on the first two steps of the proof of \cite[Theorem 3.7]{caragea2020neural}, but replacing Barron- by $RBV^2$ functions. 
\cite[Theorem 3.7]{caragea2020neural} obtains an approximation rate of $\mathcal{O}(N^{-\alpha/2})$ for $N \to \infty$, for Horizon functions associated with Barron functions, which is slightly slower than that for horizon functions with $RBV^2$ decision boundaries.
\begin{theorem}\label{thm:HorizonFunctionsApproximationTheorem}
  Let $d \in \N_{\geq 2}$, $N \in \N$, $q,C > 0$, and $\alpha \in (0,1]$. Further let, $h \in H_{\mathcal{B}}$, where $\mathcal{B} \coloneqq \{f \in RBV^2(B_1^{d-1}) \colon f(0)=0, \quad \Vert f \Vert_{RBV^2(B_1^{d-1})} \leq q\}$.

  \noindent There exists a NN $I_N$ with two hidden layers
  such that for each tube compatible measure $\mu$ with parameters $\alpha,C$, we have
 \begin{align} \label{eq:thisWeWantToProve}
    \mu
    (
      \{
        x \in B_1^d : h (x) \neq I_N(x)
      \}
    )
    \lesssim_d C q^\alpha N^{-\alpha \frac{d+3}{2d}}.
 \end{align}
Moreover, $0 \leq I_N(x) \leq 1$ for all $x \in \R^d$
  and the architecture of $I_N$ is given by
  \[
    \mathcal{A}
    = \big(
        d, \,\, N+2, \,\, 2,  1
      \big)
    .
  \]
  Thus, $I_N$ has at most $d+N+5$ neurons and at most $(d+3)N + 2d+ 11$ non-zero weights.
  The weights (and biases) of $I_N$ are bounded in magnitude by
  $\max \{1, \lceil \sqrt{2} q N^{\frac{d+3}{2d}} \rceil\}$.
\end{theorem}
\begin{proof}
Since $h \in H_\mathcal{B}$, there exists $f \in \mathcal{B}$ such that $h(x) = \mathds{1}_{x_i \leq f(x^{[i]})}$. We have by Proposition \ref{coro:OurCor}, that for all $N \in \N$, there exist $w^{(k)} \in \R^d$, $ b_k, v_k \in \R$ for $k =1,2, \dots, N$,  $c\in \R^d$ and $c_0 \in \R$ with 
\begin{align*}
    |v_k|, |b_k|, |c_0|, \|c\|, \Vert w^{(k)} \Vert \leq 5 \Vert f \Vert_{RBV^2},
\end{align*} 
such that for 
$$
    f_N(x) = \sum_{k=1}^N v_k \varrho((w^{(k)})^\top x - b_k) + c^\top x + c_0,
$$
it holds that 
\begin{align}\label{eq:weUseThisEstimateAndTheConstantInItBelow}
    \|f - f_N\|_{L^\infty(B_1^{d-1})} \lesssim_d  \Vert f \Vert_{RBV^2} N^{-\frac{d+3}{2d}}.
\end{align}

\noindent We define for $1\geq \delta >0$ the one layer NN  
$$
    H_\delta(z) \coloneqq \frac{1}{\delta} \cdot \left(\varrho(z) - \varrho(z-\delta) \right) \text{ for } z \in \R
$$ 
and for $x \in \R^d$
\begin{align*}
    h_N^{\delta}(x)&\coloneqq H_\delta(f_N(x^{[i]}) - x_i) = H_\delta\left(f_N(x^{[i]}) - \varrho(x_i) + \varrho(-x_i)\right),\\
    h_N(x) &\coloneqq \mathds{1}_{\R^+}(f_N(x^{[i]}) - x_i). 
\end{align*}
Note that $h_N^{\delta}$ is a NN with two hidden layers, architecture 
  \[
    \mathcal{A}
    = \big(
        d, \,\, N+2, \,\, 2,  1
      \big)
    ,
  \]
  and all weights bounded by $\max\{ 1/\delta,  5 \Vert f \Vert_{RBV^2} \}$. We choose $\delta \coloneqq N^{-\frac{d+3}{2d}}$ and proceed to prove \eqref{eq:thisWeWantToProve} for $I_N = f_N^{\delta}$. Towards this estimate, we define 
  \begin{align*}
        S \coloneqq \{ x \in  B_1^d \colon h(x) = 1\} \text{ and } T \coloneqq \{ x \in  B_1^d \colon h_N = 1\}.
  \end{align*}
  We have that
\begin{align*}
    \left\{
        x \in B_1^d : h (x) \neq  h_N^{\delta}(x)
      \right\} &\subset \left\{  x \in B_1^d : 0 < h_N^{\delta}(x) < 1 \right\} \cup \left\{  x \in B_1^d : h_N^{\delta}(x) = 1,  h (x) = 0 \right\}\\
      &\qquad  \cup \left\{  x \in B_1^d : h_N^{\delta}(x) = 0,  h (x) = 1 \right\}\\
      &= \left\{  x \in B_1^d : 0 < h_N^{\delta}(x) < 1 \right\} \cup \left\{  x \in B_1^d : h_N(x) = 1,  h (x) = 0 \right\}\\
      &\qquad \cup \left\{  x \in B_1^d : h_N(x) = 0,  h (x) = 1 \right\}\\
      &\eqqcolon  \left\{  x \in B_1^d : 0 < h_N^{\delta}(x) < 1 \right\}  \cup S \Delta T.
\end{align*}
We observe that by \eqref{eq:weUseThisEstimateAndTheConstantInItBelow}, for $\gamma \sim_d RTV^2(f)$ 
\begin{align*}
&S \Delta T \\
&=\left\{ x \in B_1^d \colon f(x^{[i]}) < x_i \leq f_N(x^{[i]})\right\} \cup  \left\{ x \in B_1^d \colon f_N(x^{[i]}) < x_i \leq f(x^{[i]})\right\}\\
&\subset \left\{ x \in B_1^d \colon 0 \leq f_N(x^{[i]}) - x_i < \gamma N^{-\frac{d+3}{2d}}\right\} \cup \left\{ x \in B_1^d \colon -\gamma N^{-\frac{d+3}{2d}} \leq f_N(x^{[i]}) - x_i < 0\right\}\\
&\subset  \left\{ x \in B_1^d \colon |f_N(x^{[i]}) - x_i| \leq \gamma N^{-\frac{d+3}{2d}}\right\}.
\end{align*}
In addition,
\begin{align*}
    \left\{  x \in B_1^d : 0 < h_N^{\delta}(x) < 1 \right\} \subset \{  x \in B_1^d : |f_N(x^{[i]}) - x_i| \leq \delta\}.
\end{align*}
We conclude by the $\alpha, C$ tube-compatibility of $\mu$ that
\begin{align*}
    &\mu(\left\{
        x \in B_1^d : h (x) \neq  h_N^{\delta}(x)
      \right\})\\
      &\qquad\leq \mu\left(\left\{  x \in B_1^d : 0 < h_N^{\delta}(x) < 1 \right\}\right) + \mu(S \Delta T)\\
      &\qquad\leq \mu\left(\{  x \in B_1^d : |f_N(x^{[i]}) - x_i| \leq \delta\}\right) + \mu\left(\{ x \in B_1^d \colon |f_N(x^{[i]}) - x_i| \leq \gamma N^{-\frac{d+3}{2d}} \}\right)\\
      &\qquad\leq C \cdot \left(\delta^\alpha + \gamma^\alpha N^{-\alpha\frac{d+3}{2d}}\right).
\end{align*}
Due to the choice of $\delta = N^{-\frac{d+3}{2d}}$, this completes the proof.
\end{proof}

\subsection{Estimation bounds}

Now, we state our result regarding upper bounds for the estimation problem. 
We begin by introducing two concepts: the Hinge loss for a target concept and the empirical risk minimizer.

Although Theorem \ref{thm:HorizonFunctionsApproximationTheorem}, showed a bound for the 0-1 loss, this is not a useful error measure in practice, due to its lack of continuity. 
Instead, we formulate our learning bounds with respect to the Hinge loss, which is significantly more practical in applications. 
As we consider horizon functions $h:[0,1]^d \to \{0,1 \} $ as our target classifiers and the range is not $\{-1,1 \}$ as would be typical for the Hinge loss, we present a slightly different definition of Hinge function. 

\begin{definition}\label{def:HingeRisk}
  Define $\phi \colon \R \to \R$, $\phi(x) \coloneqq \max \{0, 1-x\}$. 
    Let $d \in \N$ and let $\mu$ be a Borel measure on $B_1^d$.
  For a (measurable) target concept $h^* \colon B_1^d \to \{0, 1\}$, we define the \emph{Hinge risk} of a (measurable) function $h \colon B_1^d \to [0, 1]$ as
	\begin{align*}
		\mathcal{E}_{\phi, \mu, h^*}(f)
    \coloneqq \mathbb{E}_{X \sim \mu}
                \Big[
                  \phi
                  \Big(
                    \bigl(2 h^*(X) - 1\bigr)
                    \cdot \bigl(2 h(X) - 1\bigr)
                  \Big)
                \Big]
    .
	\end{align*}
 Let $\Lambda = B_1^d \times [0,1]$.  For a sample $S = (x_i,y_i)_{i=1}^m \in \Lambda^m$, $m \in \N$,
	we define the \emph{empirical $\phi$-risk} of $h \colon B_1^d \to [0, 1]$ as
	\begin{align*}
		\widehat{\mathcal{E}}_{\phi, S}(h)
    \coloneqq \frac{1}{m}
              \sum_{i=1}^m
                \phi
                \big(
                  \bigl(2 y_i - 1\bigr) \cdot \bigl(2 h(x_i) - 1\bigr)
                \big)
    .
	\end{align*} 

	 \noindent Finally, for a sample $S = (x_i,y_i)_{i=1}^m \in \Lambda^m$
	and a set $\mathcal{H} \subset \{h \colon B_1^d \to [0, 1]\}$,
	we call $h_S \in \mathcal{H}$ an empirical $\phi$-risk minimizer, if 
	\begin{align}
		 \widehat{\mathcal{E}}_{\phi, S}(h_S)
     = \min_{h \in \mathcal{H}}
         \widehat{\mathcal{E}}_{\phi, S}(h).
		\label{eq:empPhiRisk}
	\end{align}
\end{definition}
Finally, we present upper bounds for estimating horizon functions associated to $RBV^2$ functions. The proof of this theorem is similar to that of \cite[Theorem 5.7]{caragea2020neural}.

\begin{theorem}\label{thm:UpperBoundViaHingeLossMinimisation}
  Let $\kappa > 0$, $d\in \N_{\geq 4}$. There is a constant $\tau \geq 1$ depending on $d$ only such that the following holds:
	Let $h \in H_{\mathcal{B}}$, where $\mathcal{B} \coloneqq \{f \in RBV^2(B_1^d) \colon f(0)=0 \quad \Vert f \Vert_{RBV^2} \leq q\}$ and let $\lebesgue$ be a probability measure on $B_1^d$.
	For each $m \in \N$, let
	\begin{align*}
    \widetilde{N}(m)
    &\coloneqq \big\lceil \tau m^{2d/(3d+3)}\big\rceil, \\
		N(m)
    &\coloneqq \big\lceil \widetilde{N}(m) + d + 3 \big\rceil, \\
		W(m)
    &\coloneqq \big\lceil (d+4)\widetilde{N}(m) + 3\big\rceil , \\
		B(m)
    &\coloneqq \Big\lceil
                 \max \{1, \sqrt{2} q \widetilde{N}(m)^{\frac{d+3}{2d}}\}
               \,\,\Big\rceil
    .
	\end{align*}

 \noindent For each $m \in \N$, let $S$ be a training sample of size $m$;
  that is, $S = \bigl( X_i, h(X_i) \bigr)_{i=1}^m$ with $X_i \IIDsim \lebesgue$
  for $i \in \{ 1,\dots,m \}$.
  Furthermore, let $\Phi_{m,S} \in \NN_\ast \bigl(d, 2, N(m), W(m), B(m)\bigr)$
  be an empirical Hinge-loss minimizer; that is,
	\[
    \widehat{\mathcal{E}}_{\phi, S}\bigl(\Phi_{m, S}\bigr)
    = \min_{f \in \NN_\ast (d, N(m), W(m), B(m))}
        \widehat{\mathcal{E}}_{\phi, S}(f).
	\]
	Then, with $H \coloneqq \Indicator_{[1/2,\infty)}$, we have
	\begin{equation}
		\EE_S
    \Bigl[
      \EE_{X \sim \lebesgue}
        \bigl(
          \bigl|H (\Phi_{m, S} (X)) - h(X)\bigr|^2
        \bigr)
    \Bigr]
    =    \EE_S
         \big[
           \PP_{X \sim \lebesgue} \bigl(H(\Phi_{m,S}(X)) \neq h(X)\bigr)
         \big]
		\lesssim m^{-\frac{d+3}{3d+3} + \kappa} ,
    \label{eq:MainUpperBound}
	\end{equation}
  where the implied constant only depends on $d, \kappa, \tau$.
\end{theorem}
\begin{proof}
The proof is analogous to \cite[Theorem 5.7]{petersen2021optimal} with two differences, first, we choose the values 
$$
    a_m = m^{-\frac{d+3}{3d+3}}, \quad \delta_m^* = m^{-\frac{d+3}{3d+3} + \kappa}, \text{ and } \iota = \frac{2d}{3d + 3} + \kappa
$$
different from those in \cite[Theorem 5.7]{petersen2021optimal}. Second, to construct appropriate NN approximations, we use Theorem \ref{thm:HorizonFunctionsApproximationTheorem} instead of \cite[Theorem 5.3]{petersen2021optimal}. 
Note that, \cite[Theorem 5.7]{petersen2021optimal} is stated for functions defined on $[0,1]^d$. However, this is only due to the fact that it is based on \cite[Theorem 5]{kim2018fast} which, while stated for functions on $[0,1]^d$ holds for every compact subset of $\R^d$ as stated at the beginning of Section 2 of \cite{kim2018fast}. Therefore, \cite[Theorem 5.7]{petersen2021optimal} holds on $B_1^d$ equipped with a probability measure, as well.
\end{proof}

\section{Numerical Experiments}\label{sec:numeric}
In this section,  we illustrate how different normalizations of the boundary function affect the learning problem of binary classification. 
We assume that our target classifier outputs a binary label for each point $x \in \R^d$ with $d=2,3$ and $4$. 
Here, the decision boundary is a sufficiently smooth function $f:B_{d-1}(0) \to \R$.  

\subsection{Experiment set-up}
Our objective is to compare the performance of empirical risk minimization over NNs for classification problems with decision boundaries normalized with respect to various norms. To this end, we invoke the following set-up:

\begin{itemize}
    \item We first describe a base set of functions that will later be normalized to form the boundary functions to be learned. 
    For the case $d=2$, we consider the set of functions  
\begin{equation*}
    B \coloneqq \{B_1(0) \ni x \mapsto \sin( 2\pi kx) : k = 1+s/24,\text{ }  s=0,1,\dots, 24, \text{ }  x\in [-1,1] \},
\end{equation*}
where for each $f \in B$ we define $|f| \coloneqq |k| $. Now, for $d=3$, the functions to be learned belong to the set
\begin{equation*}
    B \coloneqq  \{B_2(0) \ni ( x,y) \mapsto \sin( 2\pi kx) \sin(2 \pi l y) : k,l = 1+ s/5, s=0,1, \dots, 5 \text{ } x\in [-1,1] \},
\end{equation*}
and for $f \in B$ we denote $|f|= |k|+|l|$. Finally, for $d=4$, we learn the functions
\begin{equation*}
    B \coloneqq  \{B_3(0) \ni( x,y, z) \mapsto  \sin( 2\pi kx) \sin(2 \pi l y) \sin( 2 \pi jz) : k,l, j = 1, 4/3, 7/3, 2 \text{ } x\in [-1,1] \}.
\end{equation*}
and again, for $f \in B$ we employ the notation $|f|= |k|+|l|+|j|$.

\item We consider the following norms: first, the uniform norm
$$
\Vert f \Vert_{\infty} = \sup_{x \in B_1^{d-1} } \vert f(x) \vert,
$$
second, the $L^1$-norm
$$
\Vert f \Vert_{L^1} = \int_{B_1^{d-1}} \vert f(x) \vert dx,
$$
third, the $C^1$-norm
$$
\Vert f \Vert_{C^1} = \sup_{|k|\leq 1} \Vert f^k \Vert_{ \infty}, $$
where $k \in \N_0^d$, fourth, the Barron norm,
$$
\Vert f \Vert_{\mathrm{Barron}} = \int_{-\infty}^{\infty} | \widehat{f}(w) | \vert w \vert dw
$$
as defined in \cite[Equation 3]{barron1993universal}, where $\widehat{f}$ is the Fourier transform of the function $f$, and the Barron norm. 
As the function $f$ is defined on a bounded domain, its Fourier transform is not well defined. 
To solve this issue and leverage the properties of the functions in the sets $B$ to be learned, we replace the Fourier transform with the appropriate term
\begin{equation*}
    \Vert f \Vert_{\mathrm{Barron}} \approx 2 \pi |f|.
\end{equation*}
Finally, we use the $RBV^2$ norm. 
For the computation of the $RBV^2$ norm, we use different approaches. 
For the case $d=2$, we use \cite[Remark 4]{parhi2022near}. 
It is stated there that the $RBV^2([-1,1])$ space is equivalent to the set of functions of second-order bounded variation 
$$BV^2([-1,1]) = \{f:[-1,1] \to \R  : TV^2(f)< \infty  \},$$ 
where 
$$
TV^2(f) = \int_{-1}^1 \vert D^2 f(x) \vert dx
$$
is the second-order total variation of $f:[-1,1] \to \R$. 

In \cite{parhi2022near}, it is proved that $RTV^2(f) = TV^2(f)$ and the second total-variation seminorm of $f$ is well-defined for all smooth functions $f$.
Hence, for each $f$ with $f(0) = 0$, we can compute its $RBV^2$ norm by
$$
    \|f\|_{RBV^2} = TV^2(f) + |f(1)|. 
$$
For the case $d=3$ and $d=4$, we use a different approach to compute the integral term in Definition \ref{def:RBV^2norm}. To this end, we generate a  set of $20$ vectors $w_i \in S^{d-1}$ and $20$ equidistant parameters $t_i \in [-1,1]$. Next, we compute the Radon transform at these points by using randomly generated points drawn according to the uniform distribution in the hyperplane $\{z: w^Tz=t\}$. Then, by using a method of finite differences, we compute the corresponding derivatives, and we continue to compute the total variation of the function. To compute the norm, we have to evaluate the function $f$ at the origin and at the elements of the canonical basis as in Definition \ref{def:RBV^2norm}.

\item To produce the functions to be learned, we normalize the elements of the corresponding base set $B$ with respect to the uniform, $L^1$, $C^1$, Barron, and $RBV^2$ norms, which yields five sets of functions for each dimension $d$. Next, we compare perform empirical risk minimization over appropriate sets of NNs to learn the classifiers associated to these normalized function classes. 

\item The sample set consists of $m$ points $(x_i, y_i)_{i=1}^m$. 
The points $x_i$ are randomly generated points drawn with respect to uniform distribution on the set $B_{d-1}(0) \times [-2,2]$. 
To determine the value of each $y_i$, we evaluate the function $y = \mathbf{1}_{x_d \leq f(x^{[d]})}$ at each $x_i$. 
We split this sample set into two subsets: the training and the test sets. 
The first one is randomly chosen with 80\% of all samples and is used to find a NN that minimizes the empirical error. 

\item We use a three-layer NN with ReLU activation function, and the number of neurons is determined according to Theorem \ref{thm:UpperBoundViaHingeLossMinimisation} with $\tau = 1$. 
We use the Hinge function as our loss function and the Adam optimizer. 

\item We compare the performance of the selected NN on the test sets. This error is measured in the mean squared sense. 
\end{itemize}

We record the generalization error for each function $f \in B$ and number of samples $m$ and present the results in Figure \ref{fig:2d}. In the $x$-axis, we register the number of samples, and in the $y$-axis, the average test error. In each figure, we plot the error for the five different norms introduced above.

\begin{figure}[htb]\label{fig:2d} 
    \centering
    \begin{minipage}{0.32\textwidth}
    \includegraphics[width = \textwidth]{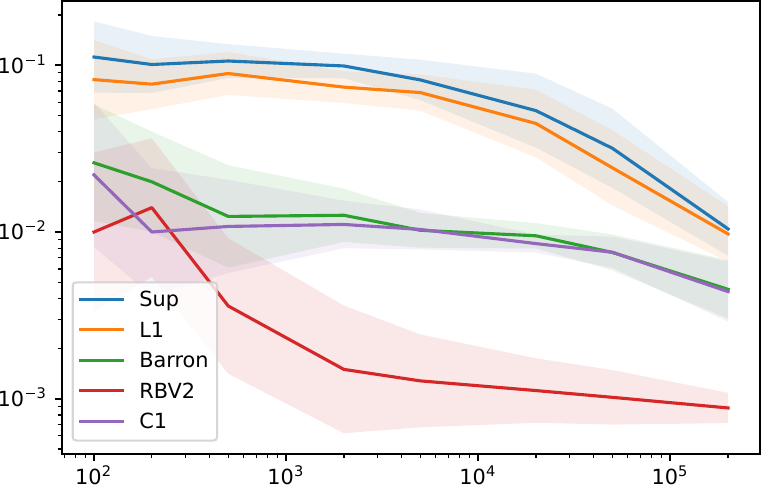}
    \end{minipage}
   \begin{minipage}{0.32\textwidth}
    \includegraphics[width = \textwidth]{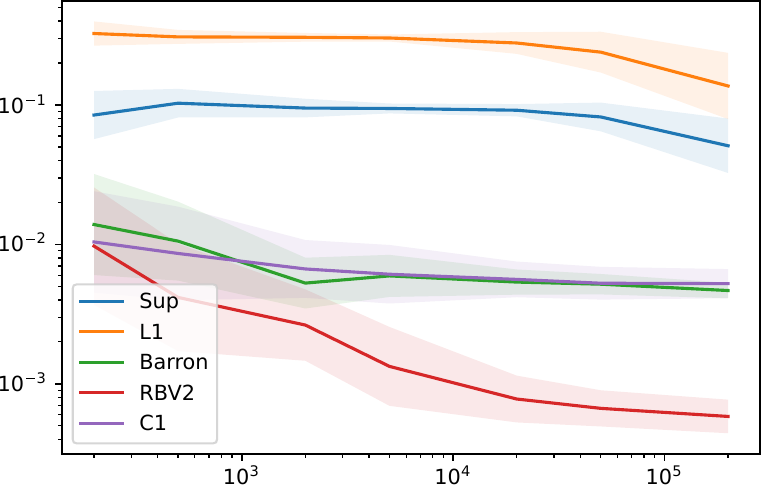}
    \end{minipage}
    \begin{minipage}{0.32\textwidth} 
    \includegraphics[width = \textwidth]{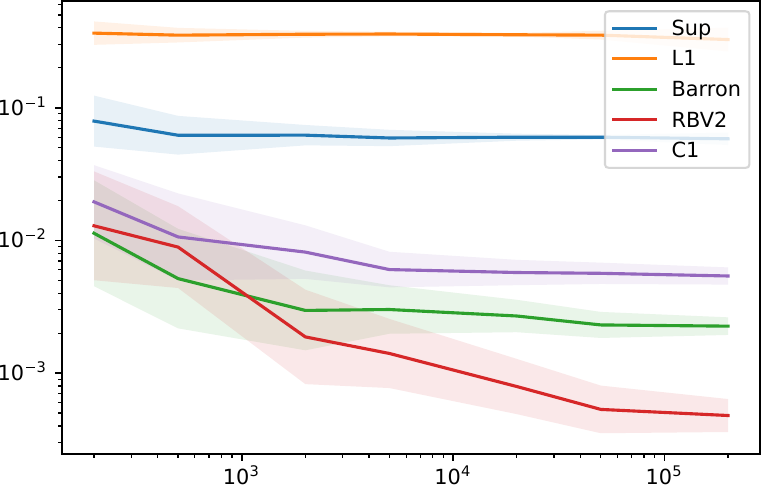}
    \end{minipage}
  \caption{ Plot of the mean relative test error for
the sets of case $d=2$ (left), $d=3$ (middle) and $d=4$ (right). The number of samples varies along the horizontal axis and the mean relative test error is shown on the vertical axis.}
\end{figure}

\subsection{Evaluation and interpretations of the results}
There are a few points we would like to highlight that can be observed from the figures. 
\begin{enumerate}
    \item In Figure \ref{fig:2d}, we can observe that when the classification function boundary is normalized with respect to $L^1$ or $L^{\infty}$ normalized, then the test error is usually the highest. Interestingly, in the case of $d=2$, the test error behaves more favorably compared to higher dimensions. For dimensions larger than two, increasing the number of samples results in almost no improvement in classification performance. This finding is consistent with the lower bounds on learnability discussed by \cite{petersen2021optimal}, which show that classification functions with boundaries in normed spaces characterized by large packing numbers are difficult to learn effectively.

    \item For the case $d=2$ and $d=3$ in \ref{fig:2d}, we observe that learning classification functions normalized using either the $C^1$ norm or the Barron norm exhibit similar performance. However, when the dimension increases to $d=4$, the Barron norm yields better test error results, which suggests that its advantages become more evident in higher dimensions.  

    \item In all cases, the classification functions with $RBV^2$ normalized boundary shows the best performance among all the tested norms. Moreover, the test errors show essentially the same decay for all dimensions tested. These examples support the idea that functions with $RBV^2$ boundary are easier to learn than functions with lower regularity and are in line with the theoretical results of Theorem \ref{thm:UpperBoundViaHingeLossMinimisation}.

    \item We can observe that the performance of $RBV^2$ boundaries surpasses that of Barron spaces.
    This is consistent with our theoretical results from Theorem \ref{thm:UpperBoundViaHingeLossMinimisation} and aligns with the findings in \cite{petersen2021optimal}, further supporting the fact that higher regularity in function boundaries leads to a better learning outcome.    
    
\end{enumerate} 
\section*{Acknowledgement}
P.~P.\ was supported by the Austrian Science Fund (FWF) [P37010]. T.~L.\ was supported by the AXA Research Fund. 
\bibliographystyle{abbrv}
\small
\bibliography{References}

\end{document}